\documentclass{article}
\usepackage[preprint]{neurips_2023}
\makeatletter
\renewcommand{\@noticestring}{Preprint.}
\makeatother

\usepackage[utf8]{inputenc}
\usepackage[T1]{fontenc}
\usepackage{hyperref}
\usepackage{url}
\usepackage{booktabs}
\usepackage{amsmath,amssymb}
\usepackage{graphicx}
\usepackage{multirow}
\usepackage{tikz}
\usepackage{enumitem}
\usetikzlibrary{arrows.meta,positioning}

\def\ba{{\mathbf a}}
\def\bx{{\mathbf x}}
\def\bz{{\mathbf z}}
\newcommand{\E}{\mathbb{E}}
\newcommand{\Prob}{\mathbb{P}}
\newtheorem{proposition}{Proposition}

\newenvironment{proof}[1][Proof]{\noindent{\it #1.}}{\hfill$\square$\medskip}

\title{A Patient World Model for Early Forecasting of\\
Digital Health Campaign Outcomes: Capabilities and Limits}

\author{%
  Yunlong Wang\\
  Advanced Analytics, IQVIA\\
  Wayne, PA, USA\\
}

\begin{document}

\maketitle

% version_3: NeurIPS-style tightening (WMHS @ NeurIPS 2026 target).
% Changes vs version_2, per the agreed 9-point plan:
%  1. intro cut to 4 paragraphs + compact contributions (architecture first)
%  2. related work cut to 3 paragraphs (world models / WM for health /
%     treatment response); campaign measurement folded into intro,
%     survival analysis folded into the method section
%  3. problem formulation merged into the method section
%  4. architecture framed as the contribution; f_theta = swappable module,
%     GRU = this paper's instantiation; new Fig. 1
%  5. analysis explicitly about the formulation, not the sequence module
%  6. data description compressed; baselines + metrics merged
%  7. results 6.1 compressed
%  8. guardrail section rewritten as a finding (user reviews at the end)
%  9. conclusion short and positive
% Style: short sentences. No semicolons or em-dashes in prose.
% Slots preserved: % ==== EXP4 SLOT ==== and % ==== EXP5 SLOT ====

\begin{abstract}
Digital direct-to-consumer (DTC) health campaigns are usually measured
after the fact. In-flight forecasting commonly relies on a separate
classifier for every cutoff and horizon. We treat this task as a
dynamic-system problem and build a compact patient world model. The
architecture maintains a latent state per patient, learns
exposure-conditioned state dynamics jointly with a weekly conversion
hazard, and rolls forward into future conversion curves.
We evaluate it on a US campaign dataset with 147{,}173 patients and
5.2 million at-risk person-weeks. In a retrospective evaluation
conditioned on recorded future exposures, the model forecasts the
remaining new-to-brand prescription volume through week 52 with a
relative error of 2.9\% from a week-4 cutoff and 0.8--2.6\% from
cutoffs at weeks 8--26. The strongest non-recurrent baseline, a
pooled-hazard gradient boosting model given the same survival rollout
and information, has relative errors of 13.6--33.1\%.
Per-horizon classifiers perform substantially worse.
A Fisher-information analysis motivates dense next-exposure supervision
when conversions are rare. Removing this auxiliary objective increases
prescription-volume error by approximately $2$--$14\times$, while
providing no consistent disadvantage on the more common specialist-visit
outcome.
We also evaluate scenario simulation. Switching all future exposure off
raises predicted conversion from 0.31 to 0.89, a pattern consistent
with selection effects in observational exposure data.
This result highlights the limits of interpreting exposure-conditioned
rollouts causally.
\end{abstract}

\section{Introduction}
\label{sec:intro}
Digital direct-to-consumer (DTC) campaigns for prescription drugs produce
patient-level exposure records. A patient may see campaign ads repeatedly
across channels and ad formats over many weeks. The outcome of interest,
such as a specialist visit or a new prescription, may occur months later.
Measurement today is mostly retrospective. Attribution and media-mix
models \citep{shao2011,zhang2014,jin2017} are fit after outcomes have
matured. Campaign teams cannot wait that long. Budgets, channel mix, and
creative rotation must be adjusted mid-flight, using partial data.
The common in-flight practice is to select a cutoff week $\tau$ and a
horizon $H$, train a classifier for that pair, and repeat for every pair
the business needs. This practice has two problems. First, nothing
enforces consistency across models trained for different horizons, so
the assembled cumulative conversion curve can decrease, which is
impossible for a true cumulative curve. Second, none of these classifiers
can predict what would happen under a different future exposure plan.
They contain no exposure-conditioned dynamics to simulate.
World models offer the missing structure. A world model jointly learns
how a latent state evolves under actions and predicts the associated
outcomes, then rolls forward to simulate future trajectories
\citep{sutton1991,ha2018}. This framework now spans control
\citep{hafner2020,hafner2025}, interactive environments \citep{bruce2024},
video simulation \citep{brooks2024,assran2025}, and driving
\citep{hu2023}. Recent position papers identify world models as the next
frontier of machine intelligence \citep{lecun2022,li2025}.
Patient world models have also begun to emerge, including Delphi-2M
\citep{shmatko2025} and ETHOS \citep{renc2024}. These models generate
plausible future health trajectories conditioned on patient history.
What they lack is an explicit representation of the intervention.
Without one, a trajectory generator cannot answer the question a
decision maker actually asks: what happens under this plan?
Our setting supplies the missing piece. Exposures are recorded at
patient-week granularity, and outcomes are dated events, allowing
forecasts to be evaluated against held-out observations at scale.

This paper develops and evaluates an intervention-conditioned patient
world model. Our primary endpoint is new-to-brand prescriptions (NBRx),
a key business outcome. This outcome is rare, making early forecasting
particularly valuable and difficult. In a retrospective evaluation
conditioned on recorded future exposures, the model forecasts the
remaining NBRx volume through week 52 with a relative error of 2.9\%
from a week-4 cutoff. Common practice produces errors exceeding
2,500\% on the same task. Even a strong pooled-hazard gradient boosting
baseline with the same survival rollout has relative errors of
13.6--33.1\% across cutoffs at weeks 4--26, compared with 0.8--2.9\%
for our model. Exposure is targeted rather than randomized, so every
forecast in this paper is exposure-conditioned rather than causal.
One of our experiments demonstrates the severe prediction reversal
that can arise when these scenario rollouts are interpreted causally.

Our contributions are as follows.
\begin{enumerate}[leftmargin=1.5em]
\item \textbf{An architecture.}
An intervention-conditioned patient world model comprising a swappable
sequence-state module (here a GRU), a survival hazard head, and an
auxiliary next-exposure prediction head. The components are trained
jointly, and recursive survival rollout produces coherent conversion
curves (Section~\ref{sec:model}).

\item \textbf{An analysis of where the gains arise.}
We establish coherence and an additive bound on cumulative-incidence
error in terms of weekly hazard errors, derive a concentration bound
for aggregate volume under explicit assumptions, and use a
Fisher-information argument to explain the potential value of dense
auxiliary supervision in the rare-event regime
(Section~\ref{sec:analysis}).

\item \textbf{Accurate early-cutoff forecasts of a rare outcome.}
NBRx volume forecasts have relative errors of 0.8--2.9\% across the
evaluated cutoffs at weeks 4--26, compared with 13.6--33.1\% for the
pooled-hazard gradient boosting baseline. Removing the auxiliary
objective increases volume error by approximately $2$--$14\times$
on this outcome, with no consistent loss on the more common
specialist-visit outcome, supporting the predicted regime dependence
(Section~\ref{sec:results}).

\item \textbf{A quantified failure of scenario simulation.}
Turning all future exposure off raises predicted conversion from
0.31 to 0.89 in our model, with a similar reversal in the baseline
simulator. We examine this artifact in relation to exposure selection
and show why these observational rollouts cannot be read as causal
responses (Section~\ref{ssec:guardrail}).
\end{enumerate}

\section{Related work}
\label{sec:related}

\paragraph{World models.} Learning environment dynamics together with the
outcome signal, then rolling the model forward, is the core pattern of
model-based reinforcement learning \citep{sutton1991,ha2018,hafner2020,
hafner2025}. The same recipe has scaled to generative interactive
environments \citep{bruce2024}, video models used as general-purpose
simulators \citep{brooks2024,assran2025}, and driving \citep{hu2023}. We
borrow the structure: latent state, action-conditioned transition, joint
training, and rollout. We do not borrow the control objective. No policy
is optimized, and our actions are observed exposures. The structure buys
us measurement: coherent forecasts at every horizon, plus scenario
simulation under stated assumptions.

\paragraph{World models for health.} Generative models of patient
trajectories are emerging fast. Delphi-2M models the natural history of
over 1{,}000 diseases from biobank-scale records \citep{shmatko2025}.
ETHOS predicts tokenized patient timelines zero-shot \citep{renc2024}.
MOTOR pretrains a time-to-event foundation model on structured medical
records \citep{steinberg2024}, and digital-twin models generate synthetic
control arms for trials \citep{fisher2019}. Surveys note that most EHR
foundation models are still evaluated far from decisions
\citep{wornow2023}. All of these model the outcome process. The
intervention process is absent or folded implicitly into the history. Our
model is far smaller, but its action stream is explicit and weekly. That
is what makes intervention-conditioned rollout and scenario simulation
well-posed tasks. We view the two lines as complementary. Trajectory
foundation models supply rich patient states. An explicit
intervention-conditioned transition turns a trajectory generator into a
decision tool.

\paragraph{Treatment-response forecasting over time.} Recurrent marginal
structural networks \citep{lim2018} and counterfactual recurrent networks
\citep{bica2020} estimate outcomes under hypothetical treatment
sequences. They inherit the assumptions of marginal structural models,
such as sequential ignorability \citep{robins2000}. We claim less. Our
model produces exposure-conditioned forecasts of observed outcomes, and
scenario rollouts are simulations under stated assumptions.
Section~\ref{ssec:guardrail} shows concretely what goes wrong when this
distinction is ignored.

\section{An intervention-conditioned patient world model}
\label{sec:model}

\subsection{Setup and tasks}
\label{ssec:problem}

Consider a campaign observed week by week. For patient $i$, let $r \in
\{0, 1, \ldots, R\}$ index the week relative to enrollment, with $R = 52$
in our data. During week $r$ the patient receives digital exposure
summarized by an \emph{action} vector $\ba_{i,r} \in \mathbb{R}^{d_a}$.
It collects log-transformed impression counts by channel, ad type, and
targeting tactic, plus recency information. A static vector $\bx_i \in
\mathbb{R}^{d_s}$ holds eligibility flags, source of business, and
baseline activity. The word action means recorded observational exposure.
Delivery is targeted, not randomized. Every quantity in this paper is
therefore an exposure-conditioned forecast, not a causal effect.
Section~\ref{ssec:guardrail} makes this concrete.

Let $T_i$ be the week of the conversion event and let $C_i$ be the
censoring week, set by the end of the observation window. Following the
discrete-time survival convention \citep{cox1972,tutz2016}, define the
per-week hazard
\begin{equation}
h_{i,r+1} \;=\; \Prob\!\left(T_i = r{+}1 \,\middle|\, T_i > r,\;
\bx_i,\, \ba_{i,1:r}\right),
\label{eq:hazard}
\end{equation}
the probability that a patient at risk through week $r$ converts in week
$r{+}1$. A patient contributes one at-risk observation for every week $r <
\min(T_i, C_i)$. The cumulative incidence through week $k$ is
\begin{equation}
F_i(k) \;=\; 1 - \prod_{m=1}^{k}\left(1 - h_{i,m}\right),
\label{eq:cuminc}
\end{equation}
which never decreases in $k$ as long as the hazards are valid
probabilities. Neural survival models adopt the same convention
\citep{katzman2018,lee2018,giunchiglia2018}, but they treat the covariate
process as given. That choice supports risk scoring and rules out
action-conditioned simulation.

Three tasks are studied. \textbf{(T1) One-step hazard prediction.} Given
the history through week $r$, predict $h_{i,r+1}$. This is a foundation
check. A model that cannot learn the one-step signal has nothing to roll
out. \textbf{(T2) Early trajectory forecasting.} Given only the history
through a cutoff $\tau < R$, forecast the future curve $F_i(k)$ for every
patient still at risk, then aggregate to the campaign volume forecast
\begin{equation}
\widehat{N}(k) \;=\; \sum_{i \in \mathcal{R}_\tau} w_i\,
\frac{F_i(k) - F_i(\tau)}{1 - F_i(\tau)},
\label{eq:volume}
\end{equation}
Here $\mathcal{R}_\tau$ is the at-risk set at $\tau$. The weights $w_i$
are inverse-sampling weights \citep{horvitz1952} that map the modeling
cohort back to natural prevalence. \textbf{(T3) Scenario simulation.} From
cutoff $\tau$, replace the future exposures with a hypothetical plan
$\ba'_{i,\tau+1:R}$ and recompute (\ref{eq:cuminc}). The stated
assumption is that the learned dynamics hold on the support of the
observed data. Task (T2) is where per-horizon practice breaks down, since
its per-cell probabilities share no state and need not be monotone. Task
(T3) additionally requires an explicit exposure-conditioned transition
model, which neither per-horizon classifiers nor outcome-only sequence
models have.

\subsection{Architecture}
\label{ssec:arch}

\begin{figure}[t]
\centering
\resizebox{0.92\textwidth}{!}{%
\begin{tikzpicture}[>={Stealth[length=1.6mm]}, font=\scriptsize,
  cell/.style={draw, rounded corners=1.2pt, minimum width=7.2mm,
               minimum height=4.6mm, inner sep=1pt, fill=white},
  headR/.style={draw, rounded corners=1.2pt, fill=gray!15, inner sep=1.6pt},
  headT/.style={draw, rounded corners=1.2pt, fill=gray!15, inner sep=1.6pt},
  var/.style={inner sep=1.6pt},
  note/.style={font=\tiny, text=black!60}]
% ---- module band (drawn first, sits behind the state cells) ----
\fill[blue!6, rounded corners=3pt] (0.88,-0.40) rectangle (8.62,0.40);
\draw[blue!35, dashed, rounded corners=3pt] (0.88,-0.40) rectangle (8.62,0.40);
% encoder cells (observed prefix)
\node[var]  (x)  at (-0.10,0) {$\bx_i$};
\node[cell] (g1) at (1.30,0) {$f_\theta$};
\node[cell] (g2) at (2.70,0) {$f_\theta$};
\node[var]  (d1) at (3.70,0) {$\cdots$};
\node[cell] (g3) at (4.70,0) {$f_\theta$};
% rollout cells (dashed)
\node[cell, dashed] (g4) at (6.30,0) {$f_\theta$};
\node[var]  (d2) at (7.25,0) {$\cdots$};
\node[cell, dashed] (g5) at (8.10,0) {$f_\theta$};
% actions
\node[var] (a1) at (1.30,-0.90) {$\ba_{i,1}$};
\node[var] (a2) at (2.70,-0.90) {$\ba_{i,2}$};
\node[var] (a3) at (4.70,-0.90) {$\ba_{i,\tau}$};
\node[var] (a4) at (6.30,-0.90) {$\ba'_{i,\tau+1}$};
\node[var] (a5) at (8.10,-0.90) {$\ba'_{i,k-1}$};
\draw[->] (a1) -- (g1); \draw[->] (a2) -- (g2); \draw[->] (a3) -- (g3);
\draw[->] (a4) -- (g4); \draw[->] (a5) -- (g5);
% recurrence
\draw[->] (x)  -- node[above]{$\bz_{i,0}$} (g1);
\draw[->] (g1) -- node[above]{$\bz_{i,1}$} (g2);
\draw[->] (g2) -- (d1); \draw[->] (d1) -- (g3);
\draw[->] (g3) -- node[above, pos=0.25]{$\bz_{i,\tau}$} (g4);
\draw[->] (g4) -- (d2); \draw[->] (d2) -- (g5);
% module annotation (top left) + pointer to the band
\node[note, anchor=west, align=left] at (-0.15,2.30)
  {$f_\theta$: sequence-state module};
\node[note, anchor=west, align=left] at (-0.15,2.02)
  {(this work: a GRU, swappable)};
\draw[->, black!45, line width=0.5pt] (0.45,1.86) to[bend right=18] (0.95,0.48);
% transition head; eq (6) takes (z_r, a_r), hence the curved input
\node[headT] (t2) at (2.70,0.95) {$g_{\mathcal T}$};
\draw[->] (g2) -- (t2);
\draw[->] (a2.west) to[out=170, in=200, looseness=1.6] (t2.west);
\node[var] (ahat) at (2.70,1.72) {$\hat\ba_{i,3}$};
\draw[->] (t2) -- (ahat);
\node[note, anchor=east] at (1.95,1.35) {transition head};
% hazard heads
\node[headR] (r3) at (4.70,0.95) {$g_{\mathcal R}$};
\node[headR] (r4) at (6.30,0.95) {$g_{\mathcal R}$};
\node[headR] (r5) at (8.10,0.95) {$g_{\mathcal R}$};
\draw[->] (g3) -- (r3); \draw[->] (g4) -- (r4); \draw[->] (g5) -- (r5);
\node[note, anchor=west] at (8.48,0.95) {hazard heads};
\node[var] (h3) at (4.70,1.72) {$h_{i,\tau+1}$};
\node[var] (h4) at (6.30,1.72) {$h_{i,\tau+2}$};
\node[var] (h5) at (8.10,1.72) {$h_{i,k}$};
\draw[->] (r3) -- (h3); \draw[->] (r4) -- (h4); \draw[->] (r5) -- (h5);
% product into cumulative incidence
\node[var, align=center] (F) at (6.55,2.45)
  {$F_i(k\,|\,\tau)=1-\prod_{m}(1-h_{i,m})$};
\draw[->] (h3.north) to[bend left=12] (F.west);
\draw[->] (h4.north) -- (F.south);
\draw[->] (h5.north) to[bend right=10] (F.south east);
\node[note] at (6.55,2.82) {coherent survival rollout};
% cutoff line
\draw[dashed, gray] (5.50,-1.20) -- (5.50,1.30);
\node[gray, font=\tiny, anchor=north] at (5.50,-1.28) {cutoff $\tau$};
\end{tikzpicture}}
\caption{The proposed architecture, unrolled. It is defined by three
components and one inference pattern: a swappable sequence-state module
$f_\theta$ (instantiated as a GRU in this work), a hazard head
$g_{\mathcal R}$, a transition head $g_{\mathcal T}$ for dense next-action
supervision, and the recursive survival rollout that multiplies weekly
hazards into a coherent cumulative-incidence forecast. The transition
head takes $(\bz_{i,r}, \ba_{i,r})$, shown by the curved input. During
training both heads apply at every at-risk week. The figure shows
$g_{\mathcal T}$ once on the observed prefix and $g_{\mathcal R}$ from the
cutoff onward, where forecasting happens. After the cutoff (dashed cells),
the same module runs on recorded or hypothetical exposures $\ba'$.}
\label{fig:arch}
\end{figure}
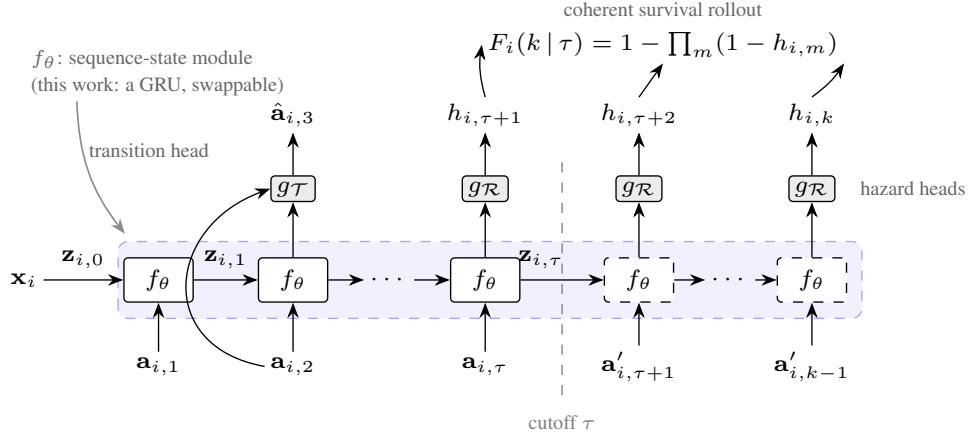

The architecture is defined by three components and one inference pattern
(Fig.~\ref{fig:arch}): a sequence-state module $f_\theta$, a hazard head
$g_{\mathcal R}$, a transition head $g_{\mathcal T}$, and a recursive
survival rollout. The sequence module is deliberately generic. We
instantiate it as a two-layer GRU \citep{cho2014} with hidden size 128 and
about $2 \times 10^5$ parameters, which trains on CPU inside the certified
environments where patient data must stay
(Appendix~\ref{app:training}). Any sequence model can take its place, and
a stronger module may improve the results. The GRU is not the
contribution. In fact, a plain GRU forecaster is one of our baselines. It
is obtained by deleting one component, and Section~\ref{sec:results}
measures exactly what that deletion costs.

The static vector sets the initial state through a learned projection,
and the module advances the state with each week's action,
\begin{equation}
\bz_{i,0} \;=\; \tanh\!\left(W_0\, \bx_i + b_0\right),
\qquad
\bz_{i,r} \;=\; f_\theta\!\left(\bz_{i,r-1},\, \ba_{i,r}\right),
\label{eq:encoder}
\end{equation}
so that $\bz_{i,r}$ summarizes what the campaign has done to patient $i$
in the first $r$ weeks. The hazard head is a small two-layer network that
outputs a hazard at every at-risk week,
\begin{equation}
h_{i,r+1} \;=\; \sigma\!\left(g_{\mathcal{R}}(\bz_{i,r})\right).
\label{eq:outhead}
\end{equation}
The transition head predicts the next week's action summary from the
current state and action,
\begin{equation}
\hat{\ba}_{i,r+1} \;=\; g_{\mathcal{T}}\!\left(\bz_{i,r},\, \ba_{i,r}\right).
\label{eq:transhead}
\end{equation}
The target of (\ref{eq:transhead}) is the observable exposure summary
itself. Transition quality can therefore be checked directly on held-out
data rather than taken on faith.

\subsection{Training objective}
\label{ssec:loss}

Let $\mathcal{A}$ be the set of at-risk patient-weeks $(i,r)$ in the
training split, and let $y_{i,r+1} \in \{0,1\}$ indicate a conversion in
week $r{+}1$. The model is trained end-to-end by minimizing
\begin{equation}
\mathcal{L} \;=\;
-\!\!\sum_{(i,r)\in\mathcal{A}}\!\! \Big[ y_{i,r+1} \log h_{i,r+1}
+ (1{-}y_{i,r+1}) \log (1{-}h_{i,r+1}) \Big]
\;+\; \lambda \!\!\sum_{(i,r)\in\mathcal{A}'}\!\!
\big\lVert \hat{\ba}_{i,r+1} - \ba_{i,r+1} \big\rVert_2^2,
\label{eq:loss}
\end{equation}
where $\mathcal{A}'$ keeps the patient-weeks whose next week is observed,
and $\lambda = 0.3$ in all experiments. The first term is the
discrete-time survival likelihood over pooled at-risk person-weeks
\citep{tutz2016}. The second term is the dynamics supervision. Setting
$\lambda = 0$ and dropping (\ref{eq:transhead}) recovers exactly a GRU
forecaster of the same capacity. The comparison between the two in
Section~\ref{sec:results} is therefore a one-flag ablation, not a
reimplementation.

\subsection{Rollout and forecasting}
\label{ssec:rollout}

Given a cutoff $\tau$, the module reads the observed prefix and produces
$\bz_{i,\tau}$. The model is then unrolled. At each future week $m >
\tau$, the recurrence (\ref{eq:encoder}) is advanced with the future
action $\ba_{i,m}$ and the hazard head outputs $h_{i,m}$. The forecast
cumulative incidence from the cutoff is
\begin{equation}
F_i(k \mid \tau) \;=\; 1 - \!\!\prod_{m=\tau+1}^{k}\!\! \left(1 -
h_{i,m}\right), \qquad k \in (\tau, R],
\label{eq:rollout}
\end{equation}
monotone and consistent across every horizon by construction. The future
actions can be the recorded exposures, a stated continuation rule, or a
hypothetical scenario $\ba'$. Campaign-level forecasts follow by the
aggregation in (\ref{eq:volume}). Each task exercises one component. Task
(T1) uses the state and the hazard head. Task (T2) adds the rollout. Task
(T3) adds the exposure-conditioned transition.

\section{Analysis: what the structure predicts}
\label{sec:analysis}

This section analyzes the formulation, not the sequence module. Every
statement below holds for any encoder $f_\theta$.
Proposition~\ref{prop:errprop} is a property of the survival rollout.
Proposition~\ref{prop:conc} is a property of the aggregation. The Fisher
argument is a property of the joint objective (\ref{eq:loss}). Proofs are
in Appendix~\ref{app:proofs}. Throughout, $\{h_{i,m}\}$ are the true
hazards and $\{\hat h_{i,m}\}$ are the model's estimates.

\subsection{Error propagation through the rollout}
\label{ssec:errprop}

One may worry that rollout errors compound multiplicatively with the
horizon. The product form (\ref{eq:rollout}) rules this out.

\begin{proposition}[Coherence and error propagation]
\label{prop:errprop}
Fix any estimates $\hat h_{i,m} \in [0,1]$ and let $\hat F_i(k|\tau)$ be
computed by (\ref{eq:rollout}). Then (i) $\hat F_i(k|\tau)$ is
non-decreasing in $k$ and $[0,1]$-valued for every realization, and (ii)
\begin{equation}
\bigl|\hat F_i(k \,|\, \tau) - F_i(k \,|\, \tau)\bigr|
\;\le\; \sum_{m=\tau+1}^{k} \bigl|\hat h_{i,m} - h_{i,m}\bigr|,
\qquad \forall k \in (\tau, R].
\label{eq:errbound}
\end{equation}
\end{proposition}

The trajectory error grows at most additively with the horizon, in
increments that are differences of small numbers. Weekly hazards average
$\bar h \approx 9\times 10^{-3}$ in our data. Part (i) costs nothing. It
holds for any hazard estimates and for any model given the
pooled-hazard-plus-rollout form, boosted trees included. The per-horizon
practice has neither property. Each $(\tau, H)$ classifier estimates
$F_i(k)$ on its own, so monotonicity can fail. A plain classifier also
has no way to use censored patients. Excluding them selects converters
into the long-horizon training slices, a bias that
Section~\ref{ssec:headline} measures at full size.

\subsection{Concentration of the campaign-volume forecast}
\label{ssec:concentration}

The next question is when the aggregate forecast (\ref{eq:volume}) can be
trusted. Condition on the cutoff-time information $\mathcal{F}_\tau$. For
each patient $i \in \mathcal{R}_\tau$, let
\begin{equation*}
q_i \;=\; \frac{\hat F_i(R \,|\, \tau) - \hat F_i(\tau)}{1 - \hat
F_i(\tau)}
\end{equation*}
be the predicted probability of converting in the remaining weeks, and let
$N = \sum_{i \in \mathcal{R}_\tau} w_i\, \mathbf{1}\{\tau < T_i \le R\}$
be the realized weighted converter count.

\begin{proposition}[Volume concentration]
\label{prop:conc}
Assume (i) conversions of distinct patients are independent conditional on
$\mathcal{F}_\tau$, and (ii) the model is trajectory-calibrated on the risk
set, i.e., $\Prob(\tau < T_i \le R \,|\, \mathcal{F}_\tau) = q_i$. Then
$\E[N \,|\, \mathcal{F}_\tau] = \widehat N(R)$, and for all $t > 0$,
\begin{equation}
\Prob\Big(\big|N - \widehat N(R)\big| \ge t \,\Big|\, \mathcal{F}_\tau\Big)
\;\le\; 2\exp\!\left(-\frac{t^2/2}{V + \tfrac{1}{3} w_{\max}\, t}\right),
\label{eq:bernstein}
\end{equation}
where $V = \sum_{i \in \mathcal{R}_\tau} w_i^2\, q_i (1-q_i)$ and $w_{\max}
= \max_i w_i$.
\end{proposition}

Proposition~\ref{prop:conc} splits the volume error into a random part and
a systematic part. The random part is bounded by (\ref{eq:bernstein}).
The systematic part is the calibration bias that breaks assumption (ii).
Three consequences are used in the experiments. Single-digit-percent
volume errors are calibration bias, not sampling noise. Per-cell errors
of hundreds of percent cannot be noise at this sample size, so they are
structural. And once two coherent models are both well calibrated, the
campaign-level aggregate stops separating them, even when patient-level
discrimination still does. The $1\sigma$ relative floor is about
$1/\sqrt{n_+}$, with $n_+$ the expected converter count. Every volume
forecast below is read against this floor.

\subsection{Dense supervision in the rare-event regime}
\label{ssec:fisher}

Finally, consider what the transition loss contributes to the shared
encoder. Fix any encoder parameter $u$ and look at the expected curvature
of the two loss terms at an at-risk week. For the hazard term, with logit
$\eta_{i,r} = g_{\mathcal{R}}(\bz_{i,r})$, the Fisher information
contribution is
\begin{equation}
\mathcal{I}^{\mathcal R}_{i,r}(u) \;=\;
h_{i,r+1}\big(1 - h_{i,r+1}\big)
\left(\frac{\partial \eta_{i,r}}{\partial u}\right)^{\!2}
\;\le\; \bar h \, \sup_{i,r} \left(\frac{\partial \eta_{i,r}}{\partial
u}\right)^{\!2}.
\label{eq:fisherR}
\end{equation}
The factor $h(1-h)$ is small because events are rare. Under a Gaussian
working model for the transition residuals with variance $\sigma^2$, the
transition term contributes
\begin{equation}
\mathcal{I}^{\mathcal T}_{i,r}(u) \;=\; \frac{\lambda}{\sigma^2}
\left\lVert \frac{\partial \hat{\ba}_{i,r+1}}{\partial u}
\right\rVert_2^2,
\label{eq:fisherT}
\end{equation}
which does not depend on $\bar h$ and appears at every observed week,
across all $d_a$ components. In our data $\bar h \approx 9 \times 10^{-3}$
for the diagnosis outcome and an order of magnitude smaller for the
prescription outcome. Per person-week, the outcome term then supplies at
most about $1\%$ of the squared-gradient scale available to the encoder,
while (\ref{eq:fisherT}) is $O(1)$. When events are rare, the next-state
loss provides most of the curvature that pins down the shared
representation.

\subsection{Three testable predictions}
\label{ssec:predictions}

\textbf{P1 (structure).} Coherence and additive error propagation belong
to the rollout, not the network. Any pooled-hazard model with the
rollout, including a plain GBM, should beat the per-$(\tau,H)$ practice
by a wide margin.

\textbf{P2 (aggregate ceiling).} Once coherent models are well
calibrated, their volume errors should cluster near the sampling floor.
Aggregates then stop separating architectures, while patient-level
discrimination still can.

\textbf{P3 (regime dependence).} The value of the learned state and of
the dense transition supervision should grow as the outcome gets rarer
and as the sequences carry more dynamics. On the common diagnosis outcome
the world model should be close to a strong pooled-hazard baseline. On
the rare prescription outcome it should pull far ahead.

The headline experiment sits exactly where P3 predicts the largest gap.
The supporting results for P1 and P2, where the world model is predicted
not to dominate and indeed does not, are in
Appendix~\ref{app:fullresults}.

\section{Experimental setup}
\label{sec:setup}

\paragraph{Data.}
\label{ssec:data}
The data come from a large US digital DTC campaign for a prescription
cardiovascular drug. The records are protected health information. They
stay on a secured analytics cluster and only aggregate metrics are
reported. From a universe of 10.9M eligible patients, a stratified
case-control cohort of 147{,}173 patients is built and split 70/15/15 by
patient hash. The pooled at-risk person-week count is 5.2M. The weekly
action vector ($d_a = 17$) holds winsorized log impression counts in
total and by channel, ad-type, and targeting buckets, an activity flag,
and the normalized relative week. The static vector ($d_s = 7$) holds
source-of-business indicators, eligibility flags, baseline visit volume,
and the enrollment week. Two outcomes are studied. The diagnosis outcome
(dx, a visit to a relevant specialist) has about 49k events in the
cohort. The prescription outcome (rx, a new-to-brand prescription) has
5{,}225 events and is roughly ten times rarer. By prediction P3, rx is
where the world model should matter most. Two design facts matter when
reading the results. The cohort is enriched to 33.6\% dx prevalence on
purpose, and every metric is reported both on the enriched cohort and
reweighted to natural prevalence with the inverse-sampling weights $w_i$.
The sequences are sparse, with a median of three active weeks per
patient. Cohort construction and full feature definitions are in
Appendices~\ref{app:cohort} and~\ref{app:features}.

\paragraph{Baselines, protocol, and metrics.}
\label{ssec:baselines}
Every model sees the same features, splits, and test patients. The
trajectory baselines are: \textbf{per-$(\tau,H)$ GBM} and
\textbf{GBM+future}, LightGBM \citep{ke2017} trained separately for each
cutoff-horizon cell, reflecting common practice, with the +future variant
also given aggregated future exposures. \textbf{Pooled-hazard GBM}, a
single discrete-time pooled-hazard LightGBM rolled out exactly like our
model. It is one coherent model for all horizons, added on purpose so
that the proposed model does not win by strawman. \textbf{GRU
forecaster}, the $\lambda = 0$ ablation. \textbf{Naive}, which
extrapolates the campaign-to-date constant hazard. The one-step task
additionally uses logistic regression, flat LightGBM, and an MLP
(Appendix~\ref{app:fullresults}). Only the proposed model is tuned, on
validation hazard likelihood. All baselines run standard library defaults
and are reported as they ran. No baseline was weakened. Campaign-level
metrics per cutoff are the relative \emph{volume error} of the final
converter count (\ref{eq:volume}) against the Kaplan--Meier
\citep{kaplan1958} count on the risk set, the \emph{trajectory
calibration} (ICI, the mean absolute gap between predicted and observed
curves), and \emph{coherence} (the fraction of monotone per-patient
curves). Patient-level metrics are AUROC, AUPRC, and calibration, both
enriched and IPW-corrected. The transition head is scored by held-out
next-state skill against mean and persistence predictors, so the claim
that it learns dynamics is checked rather than asserted.

\section{Results}
\label{sec:results}

\subsection{The headline: calling the year-end prescription volume at
week 4}
\label{ssec:headline}

\begin{table}[t]
\caption{The headline task: forecast the campaign's final NBRx volume from
an early cutoff $\tau$ (rx outcome, test split, known future exposures).
Relative error (\%) of the forecast final converter count. ``Coherent''
is the fraction of monotone per-patient curves (range across cutoffs).
Test risk sets hold 21{,}963 down to 16{,}940 patients. Observed future
converters $n_+$ = 711, 601, 489, 206, 76. At $\tau{=}39$ the $1\sigma$
sampling floor is $1/\sqrt{76} \approx 11.5\%$, so that column cannot
rank the coherent models (see text). All numbers as-run.}
\label{tab:headline}
\centering
\small
\setlength{\tabcolsep}{4.5pt}
\begin{tabular}{lrrrrrc}
\toprule
& \multicolumn{5}{c}{NBRx volume rel.\ error (\%) at cutoff $\tau$} & \\
\cmidrule(lr){2-6}
Model & $\tau{=}4$ & $\tau{=}8$ & $\tau{=}13$ & $\tau{=}26$ & $\tau{=}39$
& Coherent \\
\midrule
Per-$(\tau,H)$ GBM & 2582.0 & 3054.8 & 3754.7 & 8491.6 & 18389.9 &
19--24\% \\
Per-$(\tau,H)$ GBM + future & 2581.8 & 3054.1 & 3754.0 & 8491.9 & 18389.9
& 11--31\% \\
Naive extrapolation & 64.4 & 82.6 & 97.4 & 166.1 & 146.7 & 100\% \\
Pooled-hazard GBM & 13.6 & 17.3 & 19.2 & 33.1 & 44.3 & 100\% \\
GRU forecaster ($\lambda{=}0$) & 5.2 & 7.4 & 8.2 & 11.1 & \textbf{5.6} &
100\% \\
GRU world model (proposed) & \textbf{2.9} & \textbf{2.3} & \textbf{2.6} &
\textbf{0.8} & 23.4 & 100\% \\
\bottomrule
\end{tabular}
\end{table}

\begin{figure}[t]
\centering
\includegraphics[width=0.62\textwidth]{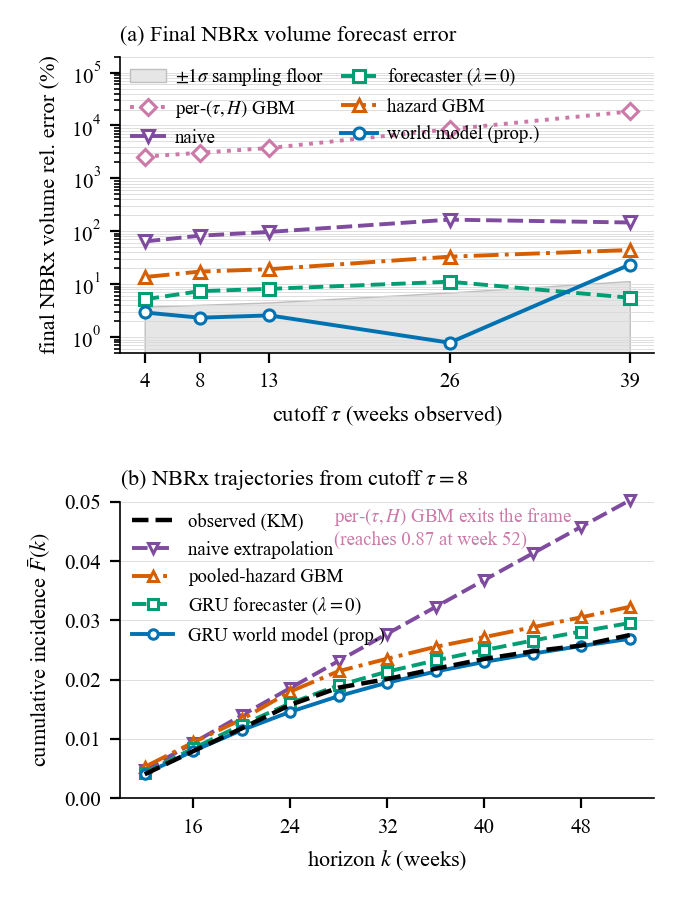}
\caption{The headline result (rx outcome, test split). (a) Relative error
of the forecast final NBRx volume versus the cutoff $\tau$ (log scale).
The shaded band is the $\pm 1\sigma$ sampling floor implied by
Proposition~\ref{prop:conc} ($\approx 1/\sqrt{n_+}$). No method, however
good, can be expected to beat it. The world model sits inside the floor
at every cutoff up to week 26. The fair coherent baseline sits
4--5$\sigma$ above it, and the per-horizon practice two to four orders of
magnitude above it. Read horizontally, a small $\tau$ means a long
rollout ($\tau = 4$ rolls 48 weeks), and the world model's error does not
grow with rollout length, as Proposition~\ref{prop:errprop} requires.
(b) Predicted NBRx cumulative-incidence trajectories from cutoff $\tau =
8$ against the observed curve, estimated by Kaplan--Meier (KM), the
standard estimator of cumulative incidence under right-censoring. The
world model tracks the observed curve across all 44 remaining weeks. The
whole curve is calibrated, not just the endpoint. The per-cell curve
leaves the frame.}
\label{fig:rx}
\end{figure}

The task: after observing the first $\tau$ weeks, forecast the number of
new-to-brand prescriptions that the at-risk audience will generate by
week 52. The evaluation feeds every model the recorded future exposures,
so the comparison isolates learned response dynamics. Forecasts are
per-patient curves (\ref{eq:rollout}) aggregated by (\ref{eq:volume}) and
scored against the Kaplan--Meier count on the same risk set.
Table~\ref{tab:headline} is the exact record. Fig.~\ref{fig:rx} plots
the same numbers against the sampling floor. At $\tau = 4$, with 711
future converters hidden in a risk set of 21{,}963 test patients, the
world model forecasts 691 conversions against 711 observed, an error of
2.9\%. From $\tau = 26$ it forecasts 208 against 206 (0.8\%). Every
curve is monotone by construction.

How close is that to the best possible? By
Proposition~\ref{prop:conc}, even a perfectly calibrated forecaster faces
a sampling floor of about $1/\sqrt{n_+}$, because the realized converter
count is itself random. On these risk sets the $1\sigma$ floor is 3.7\%,
4.1\%, 4.5\%, and 7.0\% at $\tau = 4$--$26$. The world model's errors of
2.9, 2.3, 2.6, and 0.8\% sit at or below the floor at every one of these
cutoffs. On this test set it is statistically indistinguishable from a
perfectly calibrated forecaster. There is no headroom left to measure.
The pooled-hazard GBM sits $3.6$--$4.7\sigma$ above the floor at the same
cutoffs, which is systematic bias rather than noise.

No baseline produces a usable number for this task. The per-horizon
practice errs by 2{,}582\% at week 4 and by 18{,}390\% at week 39,
forecasting about 19{,}000 conversions where 711 occur. The mechanism is
censoring. A per-cell classifier can label a non-converter at horizon $k$
only if that patient is observed through week $k$. Its long-horizon
training slices therefore keep every converter but only the
earliest-enrolled non-converters. The slices become converter-dominated,
the classifier learns a conversion rate near 0.9, and it applies that
rate to a risk set whose true rate is 0.03. The same method errs an
order of magnitude less on the ten-times-more-common dx outcome
(Appendix~\ref{app:fullresults}), which is exactly how a
selection-through-censoring bias should scale. Fewer than a third of the
assembled curves are monotone. Giving these models the aggregated future
exposures changes nothing, so the failure is structural rather than
informational. The
pooled-hazard GBM is the strong baseline built to prevent a strawman win.
It shares the survival rollout, the features, and the future exposures.
Structure alone carries it to 13.6--33.1\%, far beyond common practice,
yet an order of magnitude behind the world model at every cutoff. The
world model's trajectory ICI is also three to six times tighter, and both
GRU models add about $+0.02$ AUROC over the GBM baseline
(Appendix~\ref{app:fullresults}). Read vertically, Fig.~\ref{fig:rx}(a)
is an ablation ladder. Coupling the horizons at all is worth a factor of
40. The survival rollout is worth another 5. The learned state is worth
another 3, and the transition supervision a further $2$--$14\times$,
dissected next. Panel (b) shows that the whole curve is calibrated, so
one trained model answers every interim milestone question with the same
accuracy.

Two qualifications. First, at $\tau = 39$ only 76 future converters
remain and the floor rises to 11.5\%. The world model's 23.4\% is about
$2\sigma$ and the two GRU variants differ by less than $1.5\sigma$, so
that column cannot rank the coherent models. A real late-cutoff effect
does exist on dx, where both GRU variants underpredict by about 15\%
while the GBM stays at 6\%. It hits both values of $\lambda$ equally, so
it is not about the transition loss. Appendix~\ref{app:fullresults}
analyzes it. Second, this is a retrospective evaluation with known future
exposures, applied equally to every model.
Section~\ref{ssec:operational} discusses the step to a fully operational
forecast.

% ==== EXP5 SLOT (rarity sweep; exp5_rarity running) ==================
% When exp5 lands (experiments_results/exp5_rarity/), insert after the next
% paragraph: a figure (gap ratio vs controlled dx prevalence, WM vs
% gbm_hazard and WM vs forecaster, with the rx point overlaid at its natural
% prevalence) plus roughly this text:
%   "Prediction P3 can also be tested by experiment rather than by
%   comparing two different outcomes. Thinning the dx converters by a
%   deterministic hash makes the same outcome progressively rarer
%   (33.6% -> 10% -> 3.4% -> 1.1% cohort prevalence) while leaving the
%   exposure process untouched. Fig.~X shows the result: the volume-error
%   gap between the world model and both baselines grows monotonically as
%   events become rarer, and the rx result falls on the same curve at its
%   own prevalence. The transition skill stays flat across levels (+11.5%),
%   confirming that only the outcome signal, not the dynamics signal, was
%   manipulated."
% If exp5 contradicts the prediction, this subsection must say so and the
% mechanism claim must be weakened to the ablation evidence only.
% =====================================================================

\subsection{Why the world model wins here, and only here}
\label{ssec:mechanism}

The headline number would be suspicious if the world model won
everywhere. Strong methods usually have a regime, and the analysis
predicted this one in advance. Two checks confirm the mechanism.

First, the gap appears exactly where P3 puts it. On the diagnosis
outcome, with ten times more events on the same patients and pipeline,
the three coherent models land in the same single-digit band. The world
model reaches 4.5--7.5\% for $\tau \le 26$ and the GBM baseline
6.0--9.2\%, which is the aggregate ceiling of P2. The per-cell practice
still fails by 209--1573\%, as P1 predicts. The full dx tables, the
one-step task, and a dynamics-rich subgroup analysis are in
Appendix~\ref{app:fullresults}. The decisive advantage is confined to the
rare outcome. That is what the Fisher argument requires. With $\bar h$ an
order of magnitude smaller, the outcome likelihood alone leaves the
encoder loosely determined, so the models that share the dense
supervision should separate from those that do not.

Second, the one-flag ablation isolates the ingredient. Setting $\lambda =
0$ removes exactly one thing, the transition loss. On dx the ablation
costs nothing, and forecaster and world model are statistically
inseparable there. On rx it costs $2$--$14\times$ in volume error
(5.2/7.4/8.2/11.1\% against 2.9/2.3/2.6/0.8\% at $\tau = 4$--$26$), with
discrimination essentially tied. One flag, one order of magnitude, only
in the rare-event regime, in the direction predicted by
(\ref{eq:fisherR})--(\ref{eq:fisherT}) before the experiment ran. We read
this as supportive rather than conclusive, since it is one outcome on one
campaign and the noisy $\tau = 39$ cell goes the other way. Still, the
location and the size of the effect are hard to explain by anything
except the dense supervision.

For completeness, two dead ends are reported so others do not re-run
them. A 30k-patient probe showed a 2--5$\times$ world-model advantage on
dx that vanished at full scale, a small-sample effect. A multi-step
rollout training objective made calibration worse and was removed
(Appendix~\ref{app:training}).

\subsection{From retrospective to operational forecasting}
\label{ssec:operational}

% ==== EXP4 SLOT (closed-loop / operational rollout; exp4 running) ====
% Gate: exp4a precheck green if closed-loop rx vol err (tau=4/8 mean)
% < 13.6% (the GBM baseline's KNOWN-future error). If green and exp4b
% lands (experiments_results/exp4_operational/), REPLACE this subsection's
% second paragraph with roughly:
%   "This step has now been taken. From cutoff tau, the transition head
%   generates the future exposures itself --- imagination in the
%   world-model sense --- and the hazard head rolls out on the imagined
%   sequence, with no access to the recorded future. Table~Y: the
%   closed-loop world model forecasts the final NBRx volume within X%,
%   while every baseline must be given an assumed future plan, and the
%   best such heuristic (the train-cohort mean plan) leaves the GBM
%   baseline at Z%. Even against baselines fed the TRUE future plan
%   (Table~\ref{tab:headline}), the blind world model remains more
%   accurate. A hybrid control --- the GBM baseline rolled out on the
%   world model's imagined exposures --- recovers part of the gap,
%   locating the remaining advantage in the learned state itself."
% Add: imagination diagnostics sentence (k-step transition skill, imagined
% vs actual activity distribution), and the zero-plan row tied back to the
% alpha=0 artifact of the guardrail section.
% If exp4a is yellow/red: keep this subsection as-is (it is honest and
% complete without exp4), optionally citing the precheck diagnostics.
% =====================================================================

The evaluations above feed the recorded future exposures to every model.
This is the right controlled comparison, because it isolates response
dynamics from the separate problem of knowing the future media plan. An
operational week-4 forecast does not know that plan.

The world model is the only method in Table~\ref{tab:headline} with a
mechanism for closing this gap from within. Its transition head
(\ref{eq:transhead}) predicts next week's exposure with $+11.5\%$
held-out skill over a mean predictor, so the model can roll forward on
its own predicted exposures. Every baseline must instead be handed an
assumed future plan, and the assumption is load-bearing. The natural
choice of assuming no further exposure is exactly the off-support
scenario that Section~\ref{ssec:guardrail} shows to inflate predicted
conversion by a factor of two or more. A closed-loop evaluation of this
capability is running at the time of writing and will be reported in a
revision. The present version claims only the controlled result.

\subsection{Scenario simulation and a selection artifact}
\label{ssec:guardrail}

\begin{table}[t]
\caption{Scenario simulation (T3): simulated dose-response on the dx
outcome. Each cell is the mean predicted cumulative conversion of the
test risk set when all future exposure is scaled by $\alpha$
(\emph{do}$(\alpha \cdot \ba)$), for the world-model simulator and the
perturbed pooled-hazard GBM. The $\alpha = 0$ column is a selection
artifact, not an effect (see text). The $\alpha = 4$ column is outside
the observed dose support.}
\label{tab:exp3}
\centering
\small
\setlength{\tabcolsep}{5.5pt}
\begin{tabular}{lrrrrrrr}
\toprule
& \multicolumn{7}{c}{Dose scale $\alpha$} \\
\cmidrule(lr){2-8}
Simulator ($\tau$) & 0 & 0.25 & 0.5 & 1 & 1.5 & 2 & 4 \\
\midrule
World model ($\tau{=}4$) & \emph{0.893} & 0.312 & 0.310 & 0.310 & 0.311 &
0.311 & 0.316 \\
Pooled-hazard GBM ($\tau{=}4$) & \emph{0.632} & 0.315 & 0.314 & 0.313 &
0.318 & 0.325 & 0.364 \\
World model ($\tau{=}8$) & \emph{0.708} & 0.267 & 0.265 & 0.265 & 0.266 &
0.267 & 0.270 \\
Pooled-hazard GBM ($\tau{=}8$) & \emph{0.504} & 0.273 & 0.272 & 0.271 &
0.275 & 0.279 & 0.307 \\
\bottomrule
\end{tabular}
\end{table}

The transition model also turns the forecaster into a simulator. Replace
the future exposures with a stated scenario $\ba'$, re-roll, and read off
the predicted trajectory. We exercised this with dose scaling ($\ba' =
\alpha\,\ba$, $\alpha \in [0, 4]$), per-lever shutoffs, and a placebo
control, on dx at $\tau \in \{4, 8\}$ (protocol in
Appendix~\ref{app:exp3}). Within the observed dose support the simulator
behaves sanely (Table~\ref{tab:exp3}). The dose-response is monotone and
saturating but small, about $+0.6$ percentage points for a $4\times$
dose. The placebo moves the forecast by less than $10^{-4}$, and no
single lever moves it by more than 0.003.

The instructive result is off support. Setting $\alpha = 0$ switches all
future exposure off, and on a causal reading this should produce the
lowest conversion. It produces the highest by far. Predicted conversion
jumps from 0.31 to 0.89. The mechanism is selection through censoring.
In the observed data a patient's exposure stops once they convert, so no
future exposure is strong evidence of imminent conversion. The GBM
baseline simulator shows the same inversion at 0.63, so this is a
property of observational exposure data under outcome-dependent
censoring, not a defect of one architecture. The number is worth
remembering. It is the size of the gap between what patients who look
like this do and what would happen if we did this. Turning scenario
outputs into effect estimates requires design, such as holdouts,
randomization, or explicit deconfounding \citep{robins2000,kunzel2019}.

\section{Conclusion}
\label{sec:concl}

We formulated in-flight campaign measurement as exposure-conditioned
patient-state trajectory modeling and built a compact patient world model
for it. From four weeks of observation, the model calls the campaign's
final volume on the rare business endpoint within 2.9\%, at the sampling
floor of the test set. This is a task on which standard practice is off
by orders of magnitude. The mechanism is understood rather than guessed.
A Fisher-information argument locates the value of dense transition
supervision in the rare-event regime, and a one-flag ablation confirms
both the location and the size of the effect. The same recipe, a
survival-coherent rollout from a learned intervention-conditioned state,
applies wherever longitudinal exposures meet sparse terminal outcomes.
Adherence programs, vaccination campaigns, and public-health
communication are natural next targets.

\section*{Responsible use and limitations}

This work forecasts observed health-seeking behavior under recorded
digital exposure. It does not estimate causal effects, and
Section~\ref{ssec:guardrail} quantifies why its scenario outputs must not
be read as such. Model outputs are decision support for campaign
measurement and planning. They are not used, and should not be used, for
individual-level clinical or coverage decisions. Uncertainty is reported
throughout against the sampling floor of Proposition~\ref{prop:conc}.
Forecasts at cutoffs whose floor exceeds the model differences are
presented as unrankable rather than interpreted, and a known late-cutoff
underprediction bias of the recurrent models is documented in
Appendix~\ref{app:fullresults}. All results come from one campaign, one
therapeutic area, and one country, on an intentionally enriched cohort
that is reweighted to natural prevalence at evaluation. External
validity is untested. The trajectory evaluations condition on recorded
future exposures, applied equally to every model. An operational
closed-loop variant is under evaluation
(Section~\ref{ssec:operational}). The underlying data are protected
health information analyzed inside a certified environment. Only
aggregate metrics are reported, and no patient-level output leaves that
environment. Extended limitations and deployment guidance are in
Appendices~\ref{app:fullresults} and~\ref{app:practice}.

\paragraph{Reproducibility and data availability.} The underlying data
cannot be shared. Cohort construction rules, feature definitions,
hyperparameters, and the full as-run result tables are given in
Appendices~\ref{app:cohort}--\ref{app:practice}, at a level of detail
sufficient to rebuild the pipeline on comparable data.

\bibliographystyle{plainnat}
\bibliography{refs}

\appendix

\section{Cohort construction}
\label{app:cohort}

The modeling universe contains 10.88M patients with campaign eligibility.
The stratified case-control cohort keeps: all rx converters (5{,}225), all
on-market comparator patients (4{,}399), a deterministic-hash subsample of
dx converters (44{,}917), and a deterministic-hash subsample of on-market
non-converting negatives with at least eight observed weeks (95{,}061).
The total is 147{,}173 patients. Deterministic hashing (no random seed)
makes the cohort exactly reproducible from the universe. Each patient
carries $w_i = 1/(\text{stratum sampling fraction})$, equal to 1 for the
kept-all strata. Evaluation-time inverse-probability weighting with $w_i$
restores natural prevalence \citep{horvitz1952}.

Calendar weeks are indexed as $\mathrm{wk} =
\lfloor(\text{date} - \text{2024-01-01})/7\rfloor$. The relative week is
$r = \mathrm{wk} - \mathrm{wk}_{\text{index}}$. The observation horizon is
$r \in [0, 52]$ with a fixed absolute window end, so each patient's
observed relative window is $\min(52,\; 53 - \mathrm{wk}_{\text{index}})$.
Late enrollers have shorter windows, handled as right-censoring in the
survival likelihood. After a conversion event the patient exits the risk
set. Splits are 70/15/15 by patient hash, and no patient appears in more
than one split. The dynamics-rich subgroup used in
Appendix~\ref{app:fullresults} is defined as patients with at least five
total active exposure weeks (29{,}859 patients, 20.3\%). The originally
intended definition of eight or more active weeks is nearly empty under
this cohort's sparsity. It was widened before any model comparison was
made on the subgroup.

\section{Feature definitions}
\label{app:features}

\paragraph{Per-week action vector ($d_a = 17$).} \texttt{imp\_log}
($\log(1+x)$ of the p99-winsorized total weekly impressions),
$\log(1+\cdot)$ of the weekly event count, $\log(1+\cdot)$ impression
volumes in three channel buckets (digital, programmatic, display), three
ad-type buckets (display, video, audio), seven targeting buckets
(condition, lookalike, behavioral-contextual, conquesting, retargeting,
geo-demographic, none/other), a week-active flag, and the normalized
relative week $r/52$.

\paragraph{Static vector ($d_s = 7$).} Source-of-business one-hot (three
levels: market-naive, switched from another product, continuing on brand),
rx-eligibility flag, dx-eligibility flag, $\log(1+\cdot)$ prior
office-visit count, and normalized enrollment week.

\paragraph{Survival convention.} The hazard emitted from state $\bz_{i,r}$
predicts an event in week $r{+}1$. At-risk person-weeks are pooled for
training and evaluation. Patients are censored at their observation end.

\section{Training details}
\label{app:training}

The proposed model is implemented in PyTorch and trained on CPU. The tuned
configuration (selected on validation hazard likelihood) uses two GRU
layers, hidden size 128, dropout 0.2. The base configuration is one
layer, hidden size 64. The one-step task (Table~\ref{tab:exp1}) uses the
base configuration. The rollout and simulation experiments use the tuned
configuration. The static projection $W_0$ maps the static vector to the
initial hidden state of every layer. The hazard head is $\mathrm{hidden}
\to \mathrm{hidden} \to 1$ with ReLU and dropout. The transition head
maps $(\bz_{i,r}, \ba_{i,r})$ through one hidden layer to the
17-dimensional next-week action summary. The loss weight is $\lambda =
0.3$. Optimization: Adam, learning rate $10^{-3}$,
\texttt{ReduceLROnPlateau} (factor 0.5, patience 2) on validation hazard
negative log-likelihood, early stopping (patience 10), at most 50 epochs,
batch size 512. A control run at batch size 1024 produced materially
identical results, so batch size is not a sensitive choice here. A
multi-step rollout training objective (supervising the cumulative
incidence along the rollout rather than only one-step hazards) was tried
and made calibration worse. It is not used anywhere in this paper. The
GRU forecaster ablation sets $\lambda = 0$ and removes the transition
head, with nothing else changed. On the rollout task the transition head
reaches $+11.5\%$ held-out skill over the training-mean predictor, on
both outcomes. Baselines use standard library defaults (scikit-learn for
logistic regression and MLP, LightGBM \citep{ke2017} for all GBM
variants) and are reported as-run.

\section{Full as-run results on the diagnosis outcome and subgroups}
\label{app:fullresults}

This appendix reports everything the main text compresses: the one-step
foundation task, the full dx trajectory tables, and the dynamics-rich
subgroup, all as-run.

\subsection{One-step hazard prediction (T1)}

\begin{table}[ht]
\caption{One-step hazard prediction (T1), dx outcome, full test split
(22{,}086 patients, 768{,}431 at-risk person-weeks, 6{,}856 positives).
AUROC is reported on the enriched cohort and IPW-corrected to natural
prevalence. AUPRC is on the enriched cohort.}
\label{tab:exp1}
\centering
\small
\setlength{\tabcolsep}{4.5pt}
\begin{tabular}{lccccc}
\toprule
Model & AUROC (enr.) & AUROC (IPW) & AUPRC (enr.) & Cal.\ slope & ECE \\
\midrule
Base rate & 0.500 & 0.500 & 0.009 & --- & --- \\
Logistic regression & 0.855 & 0.838 & 0.051 & \textbf{1.005} & 0.0011 \\
LightGBM (flat) & 0.872 & 0.855 & 0.059 & 0.648 & 0.0006 \\
MLP & 0.873 & 0.856 & 0.065 & 0.929 & 0.0009 \\
GRU world model (proposed) & \textbf{0.880} & \textbf{0.864} &
\textbf{0.067} & 0.957 & \textbf{0.0002} \\
\bottomrule
\end{tabular}
\end{table}

On the one-step task the world model is best on every metric, but the
margins over LightGBM and the MLP are small, as expected. With a median
of three active weeks per patient, cumulative features already capture
most of the history, and the one-step task does not exercise the rollout.
On the rx outcome (748 test positives) all models land at AUROC
0.84--0.86, within noise of one another. The one-step likelihood cannot
separate architectures there, consistent with Section~\ref{ssec:fisher}.

\subsection{Trajectory forecasting on dx (T2)}

\begin{table}[ht]
\caption{Campaign-level trajectory forecasting (T2), dx outcome, test
split, full cohort. Volume relative error (\%) at the campaign end and
coherence. Test risk sets: 20{,}396 down to 12{,}805. $n_+$ = 5{,}874
down to 631.}
\label{tab:dxvol}
\centering
\small
\setlength{\tabcolsep}{4.5pt}
\begin{tabular}{lrrrrrc}
\toprule
& \multicolumn{5}{c}{Volume rel.\ error (\%) at cutoff $\tau$} & \\
\cmidrule(lr){2-6}
Model & $\tau{=}4$ & $\tau{=}8$ & $\tau{=}13$ & $\tau{=}26$ & $\tau{=}39$
& Coherent \\
\midrule
Per-$(\tau,H)$ GBM & 209.3 & 256.0 & 326.7 & 645.1 & 1572.8 & 54--70\% \\
Per-$(\tau,H)$ GBM + future & 209.2 & 255.8 & 326.6 & 644.9 & 1572.9 &
27--72\% \\
Naive extrapolation & 97.3 & 96.8 & 101.9 & 120.3 & 140.5 & 100\% \\
Pooled-hazard GBM & 8.7 & 9.2 & 8.8 & 6.0 & \textbf{6.2} & 100\% \\
GRU forecaster ($\lambda{=}0$) & \textbf{6.6} & \textbf{5.4} &
\textbf{3.4} & 5.7 & 15.5 & 100\% \\
GRU world model (proposed) & 7.5 & 6.5 & 4.8 & \textbf{4.5} & 15.3
& 100\% \\
\bottomrule
\end{tabular}
\end{table}

\begin{table}[ht]
\caption{Patient-level discrimination and trajectory calibration on dx
(T2), full cohort. AUROC is IPW-corrected. ICI is the mean absolute gap
between predicted and KM-observed cumulative-incidence curves.}
\label{tab:dxauroc}
\centering
\small
\setlength{\tabcolsep}{3.5pt}
\begin{tabular}{lcccccccccc}
\toprule
& \multicolumn{5}{c}{AUROC (IPW)} & \multicolumn{5}{c}{Trajectory ICI} \\
\cmidrule(lr){2-6} \cmidrule(lr){7-11}
Model & 4 & 8 & 13 & 26 & 39 & 4 & 8 & 13 & 26 & 39 \\
\midrule
Per-$(\tau,H)$ GBM & .718 & .675 & .695 & .664 & .544 &
.139 & .160 & .201 & .326 & .533 \\
Naive extrapolation & .500 & .500 & .500 & .500 & .500 &
.149 & .126 & .114 & .083 & .050 \\
Pooled-hazard GBM & .912 & .908 & .903 & .896 & .891 &
.012 & .012 & .010 & .005 & \textbf{.004} \\
GRU forecaster ($\lambda{=}0$) & .935 & .930 & .924 & .917 &
.904 & \textbf{.010} & \textbf{.008} & \textbf{.006} & \textbf{.002} &
\textbf{.004} \\
GRU world model (proposed) & \textbf{.935} & \textbf{.930} &
\textbf{.925} & \textbf{.917} & \textbf{.909} & .012 & .011 & .008 &
\textbf{.002} & \textbf{.004} \\
\bottomrule
\end{tabular}
\end{table}

\begin{figure}[ht]
\centering
\includegraphics[width=0.62\textwidth]{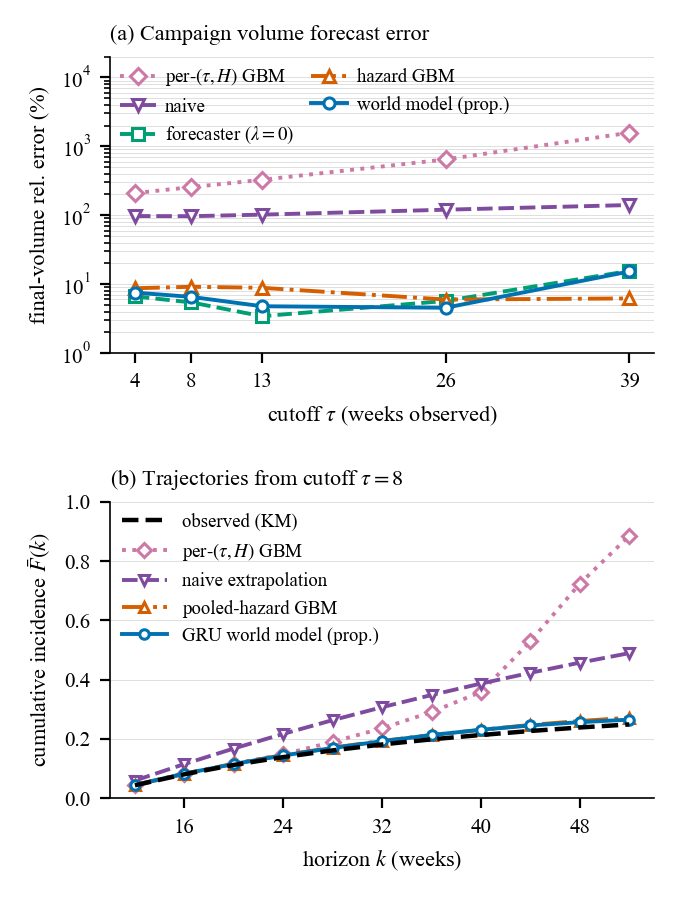}
\caption{Trajectory forecasting on dx (test split). (a) Final-volume
relative error versus cutoff (log scale). (b) Mean predicted
cumulative-incidence trajectories from $\tau = 8$ against the
Kaplan--Meier curve. The coherent models sit in the same single-digit
band. The per-cell practice explodes with the horizon.}
\label{fig:volume}
\end{figure}

Tables~\ref{tab:dxvol}--\ref{tab:dxauroc} and Fig.~\ref{fig:volume} show
the dx outcome in full. The reading is the one predicted in
Section~\ref{ssec:predictions}. The structural gap against the per-cell
practice is enormous (P1). The three coherent models are within a few
percent of each other on volume and ICI, near the $\sim$1\% sampling
floor of Proposition~\ref{prop:conc} (P2). The world model's
patient-level edge over the pooled-hazard GBM is a stable $+0.018$ to
$+0.023$ AUROC at every cutoff.

\paragraph{The late-cutoff underprediction.} At $\tau = 39$ both GRU
variants underpredict the dx volume by about 15\% while the GBM baseline
stays at 6\%. With $n_+ = 631$ the $1\sigma$ floor is only 4\%, so this
is a systematic effect of about $4\sigma$. Both values of $\lambda$ are
hit identically, so it is a property of the recurrent architecture, not
of the transition loss. Two mechanisms plausibly combine. First,
training support thins at late relative weeks. The observation window is
$\min(52, 53 - \mathrm{wk}_{\text{index}})$, so weeks $r > 39$ are
observed only for the earliest enrollees, and the hazard head sees few
at-risk examples there. Second, by week 39 most at-risk patients have
been exposure-inactive for many consecutive weeks. Under a long constant
input the recurrent state drifts toward a fixed point. Inactivity also
correlates with having already converted in the training data, the same
selection mechanism in mild form that produces the $\alpha = 0$ artifact
of Section~\ref{ssec:guardrail}, so the drifted state carries a
suppressed hazard. The GBM baseline is immune to both, because its
cumulative features simply freeze. The practical reading is that the
world model's advantage is a property of early cutoffs, which is where
early forecasting lives. A deployment would hand very late cutoffs to
the cheaper pooled-hazard model or recalibrate the tail.

\subsection{Dynamics-rich subgroup}

\begin{table}[ht]
\caption{The regime map (volume rel.\ error, \%): the three coherent
models on the full dx cohort, the dynamics-rich dx subgroup ($\ge 5$
total active weeks), and the full rx cohort. Small-$n_+$ cells are noisy.}
\label{tab:regimes}
\centering
\small
\setlength{\tabcolsep}{3.6pt}
\begin{tabular}{llrrrrr}
\toprule
Regime & Model & $\tau{=}4$ & $\tau{=}8$ & $\tau{=}13$ & $\tau{=}26$ &
$\tau{=}39$ \\
\midrule
dx, full cohort
& Pooled-hazard GBM & 8.7 & 9.2 & 8.8 & 6.0 & \textbf{6.2} \\
($n_+$: 5874--631)
& GRU forecaster & \textbf{6.6} & \textbf{5.4} & \textbf{3.4} & 5.7 & 15.5 \\
& GRU world model & 7.5 & 6.5 & 4.8 & \textbf{4.5} & 15.3 \\
\midrule
dx, dynamics-rich
& Pooled-hazard GBM & 15.0 & 15.2 & 19.5 & 21.5 & 36.1 \\
($n_+$: 599--91)
& GRU forecaster & \textbf{2.3} & \textbf{0.7} & \textbf{0.3} & 7.8 &
\textbf{16.9} \\
& GRU world model & 4.0 & 1.6 & 3.2 & \textbf{6.8} & 17.6 \\
\midrule
rx, full cohort
& Pooled-hazard GBM & 13.6 & 17.3 & 19.2 & 33.1 & 44.3 \\
($n_+$: 711--76)
& GRU forecaster & 5.2 & 7.4 & 8.2 & 11.1 & \textbf{5.6} \\
& GRU world model & \textbf{2.9} & \textbf{2.3} & \textbf{2.6} &
\textbf{0.8} & 23.4 \\
\bottomrule
\end{tabular}
\end{table}

\begin{figure}[ht]
\centering
\includegraphics[width=\textwidth]{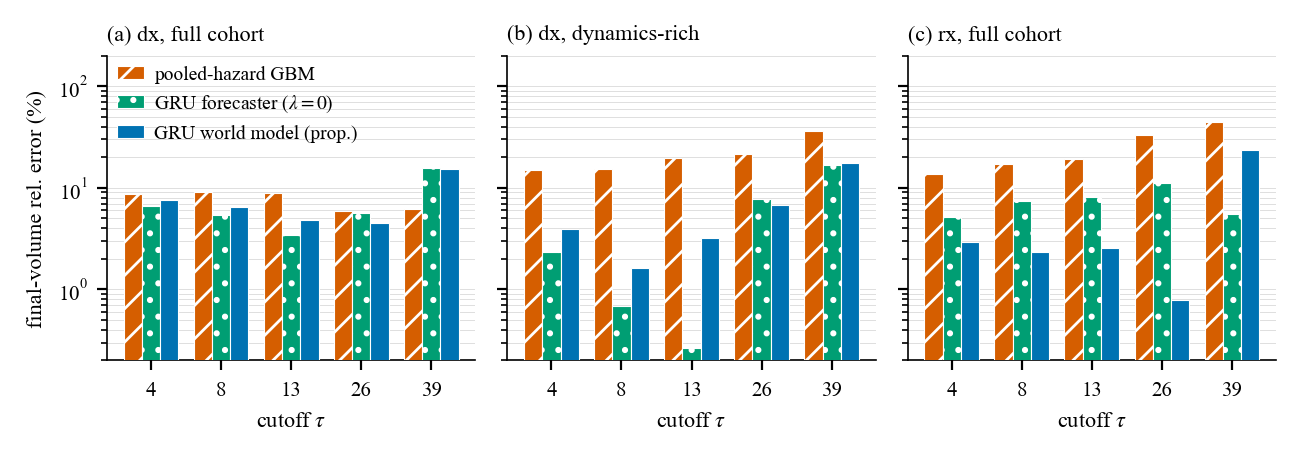}
\caption{The regime map (Table~\ref{tab:regimes}, drawn): final-volume
relative error of the three coherent models, log scale. (a) Full sparse
dx cohort: the aggregate ceiling of P2. (b) Dynamics-rich patients and
(c) the rare rx outcome: the regime dependence of P3.}
\label{fig:regimes}
\end{figure}

On patients with at least five total active weeks (20.3\% of the cohort),
the pooled-hazard GBM degrades to 15--36\% volume error while the GRU
models stay at 0.3--18\%, a $2$--$10\times$ gap, with trajectory ICI
2--3$\times$ lower and about $+0.02$ AUROC on the same patients
(Table~\ref{tab:regimes}, Fig.~\ref{fig:regimes}). Where there is a
sequence worth reading, reading it beats summarizing it. Between the two
GRU models the subgroup comparison is mixed and within noise, consistent
with the main-text finding that transition supervision separates them
only in the rare-event regime.

\section{Scenario-simulation details (T3)}
\label{app:exp3}

The simulator is the tuned world model (45 epochs, otherwise as in
Appendix~\ref{app:training}), applied to the full test risk sets at $\tau
\in \{4, 8\}$ on the dx outcome. The perturbed pooled-hazard GBM serves
as the reference simulator. Dose scaling multiplies all future exposure
magnitudes by $\alpha$. At $\alpha = 0$ this also zeroes the week-active
flag, which is what triggers the selection artifact. Per-lever shutoffs
zero one action dimension's future values (13 levers: channel, ad-type,
and targeting buckets). The largest observed shift was $+0.0026$ (total
digital volume off), and the placebo lever (a negligible bucket scaled
$4\times$) moved the forecast by less than $10^{-4}$. Per-patient
dose-response monotonicity over all $\alpha$ including 0 holds for
20--29\% of patients, dragged down by the $\alpha = 0$ artifact. Two
protocol refinements are noted for future runs: hold the week-active
pattern fixed at $\alpha = 0$ so that only magnitude varies, and restrict
reported curves to the observed dose support (approximately $\alpha \in
[0.25, 2]$).

\section{Proofs}
\label{app:proofs}

\begin{proof}[Proof of Proposition~\ref{prop:errprop}]
Part (i) is immediate, because every factor of the product in
(\ref{eq:rollout}) lies in $[0,1]$. For part (ii), write the survival
functions $\hat S(k) = \prod_{m=\tau+1}^{k}(1-\hat h_{i,m})$ and $S(k) =
\prod_{m=\tau+1}^{k}(1-h_{i,m})$, and set $P^{-}_m =
\prod_{j=\tau+1}^{m-1}(1-\hat h_{i,j})$ and $P^{+}_m =
\prod_{j=m+1}^{k}(1-h_{i,j})$. Telescoping,
\begin{equation*}
\hat S(k) - S(k) \;=\; \sum_{m=\tau+1}^{k}
P^{-}_m \big(h_{i,m} - \hat h_{i,m}\big) P^{+}_m .
\end{equation*}
Every $P^{-}_m, P^{+}_m$ lies in $[0,1]$, so the triangle inequality gives
(\ref{eq:errbound}), and $|\hat F - F| = |\hat S - S|$.
\end{proof}

\begin{proof}[Proof of Proposition~\ref{prop:conc}]
Unbiasedness follows from assumption (ii) and linearity. The variables
$w_i (\mathbf{1}\{\tau < T_i \le R\} - q_i)$ are conditionally independent
by assumption (i), zero-mean, bounded by $w_{\max}$, with variance $w_i^2
q_i (1-q_i)$. Inequality (\ref{eq:bernstein}) is Bernstein's inequality
applied to their sum.
\end{proof}

\section{Practical guidance}
\label{app:practice}

The findings compress into three pieces of advice for teams that measure
campaigns in flight.

\textbf{1. For the rare endpoint that matters, the world model is not an
increment. It is the difference between having a number and not.} A
week-4 NBRx volume call within 3\% did not previously exist at any price.
Common practice is off by orders of magnitude, and even a well-built
coherent GBM is off by 14--33\%.

\textbf{2. For common outcomes, fix the formulation and stop.} Most of
the gap between practice and the world model on the diagnosis outcome is
closed by the survival rollout alone, which any pooled-hazard model can
wear, including the GBM a team already has. The Fisher calculation of
Section~\ref{ssec:fisher} tells a team in advance which of their outcomes
justifies the sequence model.

\textbf{3. Simulate for planning, not for effect estimation.} The
simulator answers what the model expects under a stated plan, coherently
and per patient. On targeting-driven data its off-support scenarios
measure selection rather than response, and the 0.31-to-0.89 inversion is
the concrete warning. Scenario outputs need support flags and negative
controls. Effect claims need experimental design.

\end{document}